\documentclass{article}

\usepackage{arxiv}

\usepackage[utf8]{inputenc} 
\usepackage[T1]{fontenc}    
\usepackage{hyperref}       
\usepackage{url}            
\usepackage{booktabs}       
\usepackage{amsfonts}       
\usepackage{nicefrac}       
\usepackage{microtype}      
\usepackage{lipsum}
\usepackage{amsthm} 
\usepackage{graphicx}
\usepackage{doi}

\title{Interpretable transformed ANOVA approximation on the example of the prevention of forest fires}

\author{ \href{https://orcid.org/0000-0003-3651-4364}{\includegraphics[scale=0.06]{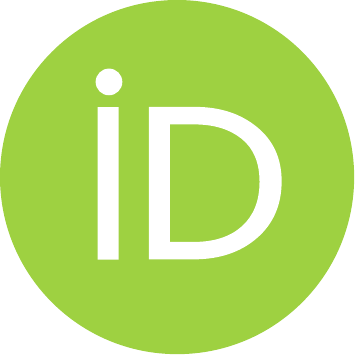}\hspace{1mm}Daniel Potts} \\
	Faculty of Mathematics\\
	Chemnitz University of Technology\\
	09107 Chemnitz \\
	\texttt{potts@math.tu-chemnitz.de} \\
	\And
	\href{https://orcid.org/0000-0003-1152-0864}{\includegraphics[scale=0.06]{orcid.pdf}\hspace{1mm}Michael Schmischke} \\
	Faculty of Mathematics\\
	Chemnitz University of Technology\\
	09107 Chemnitz \\
	\texttt{potts@math.tu-chemnitz.de} \\
}

\renewcommand{\shorttitle}{Interpretable Transformed ANOVA Approximation}

\hypersetup{
pdftitle={Interpretable transformed ANOVA approximation on the example of the prevention of forest fires},
pdfsubject={mathematics of data science},
pdfauthor={Daniel Potts, Michael Schmischke},
pdfkeywords={ANOVA, high-dimensional, approximation, interpretability, normal distribution},
}

\usepackage{amsfonts, mathtools, amssymb, amsopn, bm, mathtools}
\usepackage{mathrsfs, cleveref}
\usepackage{algorithmic}
\usepackage{booktabs, siunitx}
\usepackage{multirow}
\usepackage[dvipsnames]{xcolor}
\usepackage{tikz, pgfplots, graphicx}
\pgfplotsset{compat=1.13}
\usepackage{tikz-cd}
\usepackage{subfig}

\DeclareMathOperator*{\argmin}{arg\,min}

\DeclareMathOperator*{\supp}{supp}

\newcommand{\fun}[3]{#1 \colon #2 \rightarrow #3}
\newcommand{\abs}[1]{\left|#1\right|}
\newcommand{\norm}[2]{\left\| #1 \right\|_{#2}}
\newcommand{\transp}{^\mathsf T}

\renewcommand{\sp}[2]{\langle #1, #2 \rangle}

\renewcommand{\subset}{\subseteq}
\renewcommand{\epsilon}{\varepsilon}
\renewcommand{\b}{\bm}
\renewcommand{\cref}{\Cref}

\newcommand{\eps}{\mathrm\varepsilon}
\newcommand{\e}{\mathrm e}

\renewcommand{\i}{\mathrm i}

\newcommand{\N}{\mathbb N}
\newcommand{\R}{\mathbb R}
\newcommand{\Z}{\mathbb Z}

\newcommand{\X}{\mathcal{X}}

\newcommand{\I}{\mathcal{I}}

\newcommand{\x}{\b x}
\newcommand{\y}{\b y}
\renewcommand{\k}{\b k}

\renewcommand{\d}{\,\mathrm{d}}

\renewcommand{\L}{\mathrm{L}}

\newcommand{\D}{[d]}
\renewcommand{\u}{\b u}
\newcommand{\uc}{\b{u}^{\mathrm c}}
\newcommand{\au}{\abs{\u}}

\renewcommand{\v}{\bm v}

\newcommand{\av}{\abs{\v}}

\newcommand{\va}[1]{\sigma^2(#1)} 
\newcommand{\gsi}[2]{\varrho(#1,#2)}

\newcommand{\fc}[2]{\mathrm{c}_{#1}\!\left(#2\right)}

\newtheorem{theorem}{Theorem}[section]
\newtheorem{lemma}[theorem]{Lemma}
\newtheorem{remark}[theorem]{Remark}
\newtheorem{definition}[theorem]{Definition}
\newtheorem{example}[theorem]{Example}
\newtheorem{corollary}[theorem]{Corollary}

\begin{document}
\maketitle

\begin{abstract}
	The distribution of data points is a key component in machine learning. In most cases, one uses min-max normalization to obtain nodes in $[0,1]$ or Z-score normalization for standard normal distributed data. In this paper, we apply transformation ideas in order to design a complete orthonormal system in the $\mathrm{L}_2$ space of functions with the standard normal distribution as integration weight. Subsequently, we are able to apply the explainable ANOVA approximation for this basis and use Z-score transformed data in the method. We demonstrate the applicability of this procedure on the well-known forest fires data set from the UCI machine learning repository. The attribute ranking obtained from the ANOVA approximation provides us with crucial information about which variables in the data set are the most important for the detection of fires.
\end{abstract}


\section{Introduction}

In machine learning, the scale of our features is a key component in building models. When we work with data from applications we have to accept it as it is. In most cases, we cannot control where the nodes are lying. Let us, e.g., take recommendations in online shopping. We are only able analyze the customers that actually exist and what they bought in the shop. However, the features may lie on immensely different scales. If we measure, e.g., the time a customer spent in the shop in seconds as well as their age in years, the result will be a scale that contains values with thousands of seconds and a scale ranging from up to 90 years. Bringing those features on similar scales trough normalization may significantly improve performance of our model.

Two common methods for data normalization are min-max-normalization and Z-score normalization, see e.g.\ \cite{elements}. The former method will yield data in the interval $[0,1]$ and is especially useful if there is an intrinsic upper and lower bound for the values, e.g., when considering age. If we come back to our previous example, the time a customer spends in the shop would be less suitable since the values may have a wide range and we will probably have very few people with a significantly small or large time. In this case, the Z-score normalization makes much more sense. It tells us how many standard deviations our value lies away from the mean of the data resulting in a distribution with zero mean and variance one. 

The explainable ANOVA approximation method introduced in \cite{PoSc19a, PoSc19b, PoSc21} is based on the well-known multivariate analysis of variance (ANOVA) decomposition, see e.g.\ \cite{CaMoOw97, RaAl99, LiOw06, KuSlWaWo09, Holtz11, mcbook}, and relies on the existence of a complete orthonormal system in the space which is suitable for fast matrix-vector multiplication algorithms in grouped transformations, c.f.\ \cite{BaPoSc}. Until now, this method was always applied with min-max-normalization since it relied on the space $\L_2([0,1]^d)$ of square-integrable functions over the cube with the half-period cosine basis. It is our goal to modify the approach in order to create the possibility to work with standard normal distributed data, i.e., data that has been obtained trough Z-score normalization. 

We aim to achieve this by using the transformation ideas from \cite{KuoPaper} and \cite{NaPo18} in order to construct a complete orthonormal system in the space \begin{equation}\label{space}
	\L_2(\R^d,\omega) \coloneqq \left\{ \fun{f}{\R^d}{\R} \colon \norm{f}{\L_2(\R^d)} \coloneqq \sqrt{\int_{\R^d} \abs{f(\x)}^2 \, \omega(\x) \d\x} < \infty \right\}
\end{equation} with the probability density of the standard normal distribution \begin{equation}\label{density}
	\omega(\x) \coloneqq \prod_{i=1}^d \frac{1}{\sqrt{2 \pi}} \, \e^{-x_i^2/2} = (2\pi)^{-d/2} \, \e^{-\norm{\x}{2}^2/2}.
\end{equation} Combining this transformation with the half-period cosine basis allows for fast multiplications in the grouped transformations and makes the ANOVA approximation method applicable for Z-score normalized data.

As an example, we apply this approach to a data set about the detection of forest fires, see \cite{forest, CorMor07}. Constructing a model with the capability of efficiently predicting the size of the fire in this data set may provide a way of predicting the occurrences of fires. This creates the possibility of efficiently implementing appropriate counter-measures. In our time of climate change with massive forest fires every year, e.g., in Australia or the USA, this is an extremely current topic. With the interpretation capabilities of the ANOVA method, cf.\ \cite{PoSc21}, we are additionally able to explain the importance of our features and give reasonable explanation for the predictions.  

\section{Transformed Half-Period Cosine}\label{sec:transform}

In this section, it is our goal to construct a complete orthonormal system in the space $\L_2(\R^d,\omega)$ from \eqref{space} with the product density $\omega(\x)$ from \eqref{density}. This is the probability density function of the standard normal distribution, i.e., the normal distribution with zero mean and variance one. We have $\int_{\R^d} \omega(\x) \d\x = 1$ as well as $\sup_{\x \in \R^d} \omega(\x) = (2\pi)^{-d/2}$ which implies $\omega \in \L_{\infty}(\R^d)$. 

We aim to construct the basis using transformation ideas from \cite{NaPo18, KuoPaper} and the half-period cosine basis on $\L_2([0,1]^d)$. The orthonormal basis functions on $\L_2([0,1]^d)$ are given by \begin{equation}
	\phi^{\text{cos}}_{\k}(\x) = \sqrt{2}^{\norm{\k}{0}} \prod_{i=1}^d \cos(\pi k_i x_i), \,\k \in \N_0^d
\end{equation} with $\norm{\k}{0} \coloneqq \abs{\supp \k}$ and $\supp \k \coloneqq \{ s \in \{1,2,\dots,d\} \colon k_s \neq 0 \}$. We start from a given function $\fun{f}{\R^d}{\R}$, $f \in \L_2(\R^d,\omega)$, and aim to transform it onto the cube $[0,1]^d$. As transformation we propose \begin{equation}\label{psi}
	\fun{\psi}{[0,1]^d}{\R^d}, \, \psi(\x) = \sqrt{2} \begin{pmatrix}
		\mathrm{erf}^{-1}(2 x_1 - 1) \\
		\mathrm{erf}^{-1}(2 x_2 - 1) \\
		\vdots \\
		\mathrm{erf}^{-1}(2 x_d - 1)
	\end{pmatrix}
\end{equation} with the inverse transformation \begin{equation}\label{invpsi}
	\fun{\psi^{-1}}{\R^d}{[0,1]^d}, \, \psi^{-1}(\x) = \frac{1}{2}\begin{pmatrix}
		\mathrm{erf}(x_1/\sqrt{2})+1 \\
		\mathrm{erf}(x_2/\sqrt{2})+1 \\
		\vdots \\
		\mathrm{erf}(x_d/\sqrt{2})+1
	\end{pmatrix}.
\end{equation} The error function is given by \begin{equation*}
	\mathrm{erf}(x) = \frac{2}{\sqrt{\pi}} \int_{0}^{x} \e^{-t^2} \d t.
\end{equation*} As a result, we have the commutative diagram in \cref{fig:cd}. This allows us to transform the half-period cosine to a complete orthonormal system on $\L_2(\R^d,\omega)$ with the help of 

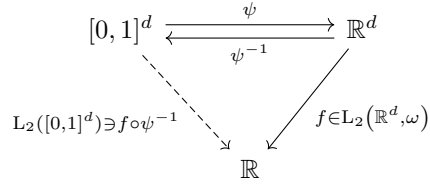
\begin{figure}[tbhp]
	\begin{center}	
	\begin{tikzcd}[row sep=huge]
		[0,1]^d \arrow[dashrightarrow, swap]{dr}{ \phantom{h(\psi(\cdot))\sqrt{\omega(\psi(\cdot))\,\psi'(\cdot)} \, =:} \L_{2}([0,1]^d) \ni
			f \circ \psi^{-1} } \arrow[shift left]{rr}{\psi} 
		& 
		& \mathbb{R}^{d} \arrow[shift left]{ll}{\psi^{-1}}\arrow{dl}{f \in 	\L_{2}\left(\mathbb{R}^d,\omega\right) \phantom{\ni\ni \ni\ni\ni\ni\ni\ni} 
		} \\
		& \mathbb{R} & 
	\end{tikzcd}
\end{center}
	\caption{Commutative diagram of the function and the transformations.}\label{fig:cd}
\end{figure}

\begin{lemma}\label{lem:spaces:help}
	Let $g, h \in \L_2([0,1]^d)$, $u, v \in \L_2(\R^d,\omega)$ with probability density $\omega$ from \eqref{density}, and transformation $\psi$, $\psi^{-1}$ as in \eqref{psi} and \eqref{invpsi} respectively. Then \begin{align*}
		\sp{g \circ \psi^{-1}}{h \circ \psi^{-1}}_{\L_2(\R^d,\omega)} &= \sp{g}{h}_{\L_2([0,1]^d)} \\
		\sp{u}{v}_{\L_2(\R^d,\omega)} &= \sp{u \circ \psi}{v \circ \psi}_{\L_2([0,1]^d)}
	\end{align*} and subsequently $\norm{h}{\L_2([0,1]^d)} = \norm{h \circ \psi^{-1}}{\L_2(\R^d,\omega)}$ and $\norm{u \circ \psi}{\L_2([0,1]^d)} = \norm{u}{\L_2(\R^d,\omega)}$.
\end{lemma}
\begin{proof}
	Let $g, h \in \L_2([0,1]^d)$ and $u, v \in \L_2(\R^d,\omega)$. Then we insert the definition and perform a change of variables as follows \begin{align*}
		\sp{g \circ \psi^{-1}}{h \circ \psi^{-1}}_{\L_2(\R^d,\omega)} &= \int_{\R^d} g(\psi^{-1}(\x)) \, h(\psi^{-1}(\x)) \, \omega(\x) \, \d\x \\
		&= \int_{\R^d} g(\b t) \, h(\b t) \, \omega(\psi(\b t)) \, \psi'(\b t) \, \d\b t.
	\end{align*} As functional determinant we obtain $\psi'(\b t) = \prod_{i=1}^{d} \sqrt{2\pi} \, \e^{\mathrm{erf}^{-2}(2x_i-1) }$ and subsequently \begin{equation*}
		\omega(\psi(\b t)) \, \psi'(\b t) = \prod_{i=1}^d \frac{1}{\sqrt{2\pi}} \e^{-\mathrm{erf}^{-2}(2x_i-1)} \cdot \sqrt{2\pi} \,  \e^{\mathrm{erf}^{-2}(2x_i-1) } = 1.
	\end{equation*} This proves the first equality. For the second equality, we way use an analogous procedure.
\end{proof}

\begin{theorem}\label{thm:system}
	The functions $(\phi_{\b k}^{\mathrm{trafo}})_{\k \in \N_0^d}$ with \begin{equation}\label{trafo}
		\phi_{\b k}^{\mathrm{trafo}}(\x) \coloneqq (\phi^{\mathrm{cos}}_{\k}\circ\psi^{-1})(\x) = \sqrt{2}^{\norm{\k}{0}} \prod_{i=1}^d \cos\left(\pi k_i \frac{\mathrm{erf}(x_i/\sqrt{2})+1}{2}\right)
	\end{equation} form a complete orthonormal system in $\L_2(\R^d,\omega)$.
\end{theorem}
\begin{proof} 
	Clearly, we have $\phi_{\b k}^{\mathrm{trafo}} \in \L_2(\R^d,\omega)$, $\k \in \N_0^d$ by \cref{lem:spaces:help}. The orthonormality follows from \cref{lem:spaces:help} and the orthonormality of the half-period cosine basis, i.e., for $\b k, \b\ell \in \N_0^d$ we have \begin{align*}
		\sp{\phi_{\b k}^{\mathrm{trafo}}}{\phi_{\b\ell}^{\mathrm{trafo}}}_{\L_2(\R^d,\omega)} &= \sp{\phi^{\mathrm{cos}}_{\k}\circ\psi^{-1}}{\phi^{\mathrm{cos}}_{\k}\circ\psi^{-1}}_{\L_2(\R^d,\omega)} \\ 
		&= \sp{\phi^{\mathrm{cos}}_{\k}}{\phi^{\mathrm{cos}}_{\b\ell}}_{\L_2([0,1]^d)} = \delta_{\b k, \b\ell}.
	\end{align*}
	It remains to show that the system is complete, i.e., for every $f \in \L_2(\R^d,\omega)$ we have \begin{equation*}
		\lim_{n \rightarrow \infty} \norm{f - \sum_{i=1}^n \sp{f}{\phi_{\k_n}^{\mathrm{trafo}}}_{\L_2(\R^d,\omega)} \, \phi_{\k_n}^{\mathrm{trafo}}}{\L_2(\R^d,\omega)} = 0
	\end{equation*} with $\k_n$, $n=1,2,\dots$, an order of $\b k \in \N_0^d$. First of all, we have \begin{equation*}
		\hat{c}_{\k} \coloneqq \sp{f}{\phi_{\k}^{\mathrm{trafo}}}_{\L_2(\R^d,\omega)} = \sp{f \circ \psi}{\phi_{\k}^{\mathrm{cos}}}_{\L_2([0,1])}
	\end{equation*} by \cref{lem:spaces:help}. We apply the norm equality from \cref{lem:spaces:help} to obtain \begin{equation*}
		\lim_{n \rightarrow \infty} \norm{f - \sum_{i=1}^n \hat{c}_{\k} \, \phi_{\k_n}^{\mathrm{trafo}}}{\L_2(\R^d,\omega)} = \lim_{n \rightarrow \infty} \norm{f \circ \psi - \sum_{i=1}^n \hat{c}_{\k} \, \phi_{\k_n}^{\mathrm{cos}}}{\L_2([0,1]^d)}
	\end{equation*} and since $(\phi_{\k}^{\mathrm{cos}})_{\k \in \N_0^d}$ is complete in $\L_2([0,1]^d)$ and $f \circ \psi \in \L_2([0,1]^d)$, the limit is zero and our statement is proven.
\end{proof}

In summary, we have constructed a complete orthonormal system $(\phi_{\b k}^{\mathrm{trafo}})_{\k \in \N_0^d}$ on the weighted space $\L_2(\R^d,\omega)$ using transformation ideas from \cite{NaPo18} and the well-known half-period cosine basis $(\phi_{\k}^{\mathrm{cos}})_{\k \in \N_0^d}$ on $\L_2([0,1]^d)$. The transformation $\psi$ in one dimension and the corresponding basis functions $\phi_{\b k}^{\mathrm{trafo}}$ are visualized in \cref{fig:visual}.

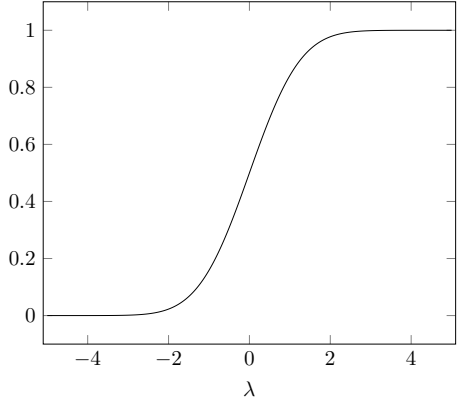
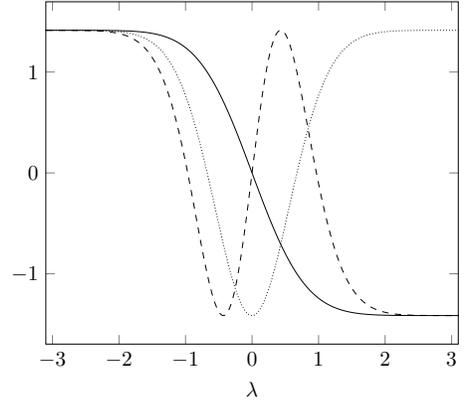
\begin{figure}[tbhp]
	\centering
	\subfloat[Transformation $\psi$ from \eqref{psi} in one dimension]{
		\begin{tikzpicture}[scale=0.8] \begin{axis}[
				xmin = -5.1, xmax = 5.1,
				xlabel = $\lambda$,
				]
				\addplot[no marks] table[x = x1, y = x2] {data.csv};
		\end{axis}\end{tikzpicture}
	} \hfill \subfloat[Basis functions $\phi_{k}^{\mathrm{trafo}}$ from \eqref{trafo} in one dimension for $k=1$ (solid), $k=2$ (dotted), and $k=3$ (dashed)]{\begin{tikzpicture}[scale=0.8] \begin{axis}[
				xmin = -3.1, xmax = 3.1,
				xlabel = $\lambda$,
				]
				\addplot[no marks, solid] table[x = x1, y = x4] {data.csv};
				\addplot[no marks, densely dotted] table[x = x1, y = x5] {data.csv};
				\addplot[no marks, dashed] table[x = x1, y = x6] {data.csv};
	\end{axis}\end{tikzpicture}}
	\caption{Transformation $\psi$ and transformed basis functions $\phi_{k}^{\mathrm{trafo}}$ in one dimension}\label{fig:visual}
\end{figure}

\section{Interpretable ANOVA Approximation}\label{sec:anova}

In this section, we briefly summarize the interpretable ANOVA (analysis of variance) approximation method and the idea of grouped transformations, see \cite{PoSc19a, BaPoSc}. The approach was considered for periodic functions, but has been expended to non-periodic functions in \cite{PoSc19b, PoSc21}. In this paper, we focus on functions $\fun{f}{\R^d}{\R}$ from $\L_2(\R^d,\omega)$ with probability density $\omega$ from \eqref{density}. Since $\omega$ is the standard normal distribution, this function space is of a high relevance. It allows us e.g.\ to work with data from applications that has been Z-transformed, i.e., data with zero mean and variance one, see e.g.\ \cite{elements}. Since the transformed half-period cosine $( \phi_{\b k}^{\mathrm{trafo}} )_{\k \in \N_0^d}$, see \cref{thm:system}, is a complete orthonormal system in the space  $\L_2(\R^d,\omega)$, we have \begin{equation}\label{baserep}
	f(\x) = \sum_{\k \in \Z^d} \fc{\k}{f} \phi_{\b k}^{\mathrm{trafo}}(\x), \quad \fc{\k}{f} = \sp{f}{\phi_{\b k}^{\mathrm{trafo}}}_{\L_2(\R^d,\omega)},
\end{equation} and through Parseval's identity $\norm{f}{\L_2(\R^d,\omega)}^2 = \sum_{\k \in \Z^d} \abs{\fc{\k}{f}}^2$.

The classical ANOVA decomposition, cf.~\cite{CaMoOw97, RaAl99, LiOw06, Holtz11}, provides us with a unique decomposition in the frequency domain as shown in \cite{PoSc19a}. We denote the coordinate indices with $\D = \{ 1,2,\dots, d\}$ and subsets as bold small letters, e.g., $\u \subset \D$. The ANOVA terms are defined as \begin{equation*}
	f_{\u}(\x) = f_{\u}(\x_{\u}) \coloneqq \sum_{\substack{\k \in \Z^d \\ \supp \k = \u}} \fc{\k}{f} \phi_{\b k}^{\mathrm{trafo}}(\x).
\end{equation*} The function can then be uniquely decomposed as \begin{equation*}
	f(\x) = \sum_{\u \subset \D} f_{\u}(\x)
\end{equation*} into $\abs{\mathcal{P}(\D)} = 2^d$ ANOVA terms where $\mathcal{P}(\D)$ is the potency set of $\D$. Here, the exponentially growing number of terms shows an expression of the curse of dimensionality in the decomposition. 

It is our goal to obtain information on how important the ANOVA terms $f_{\u}$ are with respect to the function $f$. In order to measure this, we define the variance of a function $f$ as \begin{equation*}
	\va{f} \coloneqq \norm{f}{\L_2(\R^d,\omega)}^2 - \abs{\fc{\b 0}{f}}^2 = \sum_{\k \in \Z^d\setminus\{\b 0\}} \abs{\fc{\k}{f}}^2.
\end{equation*} Note that we have the special case $\va{f_{\u}} = \norm{f_{\u}}{\L_2(\R^d,\omega)}^2$, $\u \subset \D$. The relative importance with respect to $f$ is then measured via global sensitivity indices (GSI) or Sobol indices, see \cite{So90, So01, LiOw06}, defined as \begin{equation}\label{gsi}
	\gsi{\u}{f} \coloneqq \frac{\va{f_{\u}}}{\va{f}}.
\end{equation} From the GSI we get a motivation for the concept of effective dimensions, specifically the superposition dimension as one notion of effective dimension. For a given $\alpha \in [0,1]$ it is defined as \begin{equation}\label{eq:ds}
	\mathrm{d}^{(\mathrm{sp})} \coloneqq \min \left\{ s \in \D \colon \frac{1}{\va{f}} \sum_{\substack{\u \subset \D \\ \au \leq s}} \norm{f_{\u}}{\L_2(\R^d,\omega)}^2 \geq \alpha \right\}.
\end{equation} The superposition dimension $\mathrm{d}^{(\mathrm{sp})}$ tells us that we can explain the $\alpha$-part of the variance of $f$ by terms $f_{\u}$ with $\u \leq d_s$.

Using subsets of ANOVA terms $U \subset \mathcal{P}(\D)$, it is our goal to find a way to circumvent the curse of dimensionality for efficient approximation. In order to achieve this, we aim to truncate the ANOVA decomposition by taking only the ANOVA terms in $U$ into account. The truncated ANOVA decomposition is then defined as \begin{equation*}
	\mathrm{T}_{U} f (\x) = \sum_{\u \in U} f_{\u}(\x). 
\end{equation*} A specific idea for the truncation comes from the superposition dimension $\mathrm{d}^{(\mathrm{sp})}$ in \eqref{eq:ds}. The idea is to take only variable interactions into account that contain $d_s$ or less variables, i.e., the subset of ANOVA terms is \begin{equation}\label{termset}
	U(d, d_s) \coloneqq \left\{ \u \subset \D\colon \au \leq d_s \right\}.
\end{equation} Since $d_s$ does not necessarily have to coincide to the superposition dimension $\mathrm{d}^{(\mathrm{sp})}$, we call it \emph{superposition threshold}. A well-known fact from learning theory is that the number of terms in $U(d, d_s)$ grows only polynomially in $d$ for fixed $d_s < d$, i.e., $$\abs{U(d, d_s)} \leq \left(\frac{\e d}{d_s}\right)^{d_s}$$ which has reduced the curse of dimensionality.

In the following, we argue why the truncation by a superposition threshold $d_s$ works well in relevant cases. For the approximation of functions that belong to a space $H^s(\R^d,\omega)\subset \L_2(\R^d,\omega)$ that characterizes the smoothness $s > 0$ by the decay of the basis coefficients $\fc{\k}{f}$, we can show upper bounds on the superposition dimension $\mathrm{d}^{(\mathrm{sp})}$ for $\alpha \in [0,1]$, see e.g.\ \cite{PoSc19a}. In fact, there are types of smoothness that are proven to yield a low upper bound for the superposition dimension specifically dominating-mixed smoothness with POD (product and order-dependent) weights, cf.~\cite{KuSchwSl12, GrKuNi14, KuNu16, GrKuNu18, PoSc19a}. 

In terms of real data from applications, the situation is much different. Here, we cannot make the assumption that in complete generality we have a low superposition dimension. However, there are many application scenarios where numerical experiments successfully showed that this is indeed the case, see e.g.~\cite{CaMoOw97}. Since we generally do not have a-priori information, we work with low superposition thresholds $d_s$ for truncation and validate on our test data. 

\subsection{Approximation Procedure}\label{sec:procedure}

In this section, we briefly discuss how the approximation is numerically obtained and how we can interprete the results. In this section, we assume a given subset of ANOVA terms $U \subset \mathcal{P}(\D)$. This set may be equal to or a subset of $U(d,d_s)$. We have given scattered data in the form of a set $\X = \{ \x_1, \x_2, \dots, \x_M \}\subset \R^d$ of standard normal distributed nodes and values $\y \in \R^M$, $M \in \N$. Moreover, we assume that there is an $\L_2(\R^d,\omega)$ function $f$ of form \eqref{baserep} with $f(\x_i) \approx y_i$ which we want to approximate.

First, we truncate $f$ to the set $U$ such that $f \approx \mathrm{T}_U f$. However, there are still infinitely many coefficients and therefore we perform a truncation to partial sums on finite support index sets \begin{equation}\label{indexsets}
	\I_{\emptyset} = \{0\}, \text{  and  } \I_{\u} = \{ 1, 2, \dots, N_{\au}-1 \}^{\au}
\end{equation} with order-dependent parameters $N_{\au} \in \N, \au = 1,2,\dots,d_s$, for every ANOVA term $f_{\u}$, $\u \in U$. Using the projections $P_{\u} \I_{\u} = \{ \k \in \N_0^d \colon \k_{\u} \in \I_{\u}, \k_{\uc} = \b 0 \}$, we obtain \begin{equation*}
	f_{\u}(\x) \approx \sum_{\k \in P_{\u} \I_{\u}} \fc{\k}{f} \phi_{\b k}^{\mathrm{trafo}}(\x).
\end{equation*} Now, we taking the union $\I(U) = \bigcup_{\u \in U} P_{\u} \I_{\u}$ yields $$f(\x) \approx \sum_{\k \in \I(U)} \fc{\k}{f} \phi_{\b k}^{\mathrm{trafo}}(\x).$$ The unknown coefficients $\fc{\k}{f}$ are now to be determined.

We aim to achieve this by solving the regularized least-squares problem \begin{equation}\label{eq:minimization}
	\hat{\b f} = (\hat{f}_{\k})_{\k \in \I(U)} = \argmin_{\hat{\b g} \in \R^{\abs{\I(U)}}} \norm{\y - \b{F}(\X,\I(U)) \hat{\b g}}{2}^2 + \lambda \norm{\hat{\b g}}{2}^2,
\end{equation} cf.\ \cite{PoSc19a, PoSc19b, BaPoSc}, with the basis matrix $\b{F}(\X,\I(U)) = (\phi_{\b k}^{\mathrm{trafo}}(\x))_{\x \in \X, \k \in \I(U)}$. We solve the problem using the iterative LSQR solver \cite{PaSa82}. In order to apply LSQR, we rewrite \eqref{eq:minimization} by observing the equality \begin{equation}\label{bartel}
	\norm{\y - \b{F}(\X,\I(U)) \hat{\b g}}{2}^2 + \lambda \norm{\hat{\b g}}{2}^2 = \norm{\begin{pmatrix}
			\y \\ \b 0
		\end{pmatrix} - \begin{pmatrix}
			\b{F}(\X,\I(U)) \\ \sqrt{\lambda} \b I
		\end{pmatrix} \hat{\b g}}{2}^2
\end{equation} with $\b 0$ the zero vector in $\R^{\abs{\I(U)}}$ and $\b I \in \R^{\abs{I(U)},\abs{I(U)}}$ the identity matrix. Note that we always have a unique solution in this case since the matrix $$ \begin{pmatrix}
	\b{F}(\X,\I(U)) \\ \sqrt{\lambda} \b I
\end{pmatrix} $$ has full column rank. However, the solution depends on the regularization parameter $\lambda$.

We apply the matrix-free variant of LSQR, i.e., we never explicitly construct the matrix $\b{F}(\X,\I(U))$. The grouped transformations introduced in \cite{BaPoSc} provide \emph{oracle} functions for the multiplications of $\b{F}(\X,\I(U))$ and its transposed $\b{F}\transp(\X,\I(U))$ with vectors. For our specific basis functions $\phi_{\b k}^{\mathrm{trafo}}$ the grouped transformations are based on the non-equispaced fast cosine transform or NFCT, see \cite{KeKuPo09, PlPoStTa18}. The transformation uses parallelization to separate our multiplication into smaller, up to $d_s$-dimensional NFCTs which results in an efficient algorithm. For more details we refer to \cite{BaPoSc}.

In order to solve the minimization we employ the iterative LSQR solver \cite{PaSa82} which needs a method for efficient multiplication with $\b{F}(\X,\I(U))$ and its adjoint $\b{F}^\ast(\X,\I(U))$ in the periodic case, otherwise its transposed matrix. This is realized by the Grouped Transformation idea in \cite{BaPoSc} based on the NFFT or the NFCT, see \cite{KeKuPo09, PlPoStTa18}.

One key fact is that the nodes $\X$ have to be distributed according to the probability density $\omega$ of the space such that the Moore–Penrose inverse $\b{F}^\dagger(\X,\I(U))$ is well-conditioned. In our case, $\omega$ is the density of the standard normal distribution, i.e., the nodes $\X$ have to be distributed accordingly. For a detailed discussion on the properties of those matrices we refer to \cite{KaUlVo19, Moeller2021} where our basis is a special case.

We use the global sensitivity indices $\gsi{\u}{S(\X,\I(U)) f }$, $\u \in U$, from the approximation $S(\X,\I(U)) f (\x)$ to compute approximations for the global sensitivity indices $\gsi{\u}{f}$ of the function $f$. Here, we do not consider the index to be a good approximation if the values are close together, but rather if there order is identical, i.e., we have \begin{equation*}
	\gsi{\u_1}{ f } \leq \gsi{\u_2}{f} \Longrightarrow \gsi{\u_1}{S(\X,\I(U)) f } \leq \gsi{\u_2}{S(\X,\I(U)) f } 
\end{equation*} for any pair $\u_1, \u_2 \in U$. We assume that this is the case for our choices of index sets $\I(U)$. In particular, the quality of the approximation corresponds to the accuracy of this assumption. 

In order to rank the influence of the variables $x_1, x_2, \dots, x_d$ we use the ranking score \begin{equation}\label{eq:ranking}
	r(i) = \frac{\sum_{\u \in \{\v \in U \colon i \in \v\}} \abs{\{\v \in U \colon \au = \av, i \in \v\}}^{-1} \gsi{\u}{S(\X,\I(U)) f }}{\sum_{\u \in U} \left(\sum_{i\in \u} \abs{\{\v \in U \colon \au = \av, i \in \v\}}^{-1}\right) \gsi{\u}{S(\X,\I(U)) f } } .
\end{equation} for $i=1,2,\dots,d$ which was introduced in \cite{PoSc21}. Note that this score has order-dependent weight and is normalized such that $\sum_{i \in \D} r(i) = 1$. Computing every score $r(i)$, $i \in \D$ provides an attribute ranking with respect to $U$ showing the percentage that every variable adds to the variance of the approximation. We then conclude that if we have a good approximation $S(\X,\I(U)) f $, the corresponding attribute ranking will be close to the attribute ranking of the function $f$.

\subsection{Active Set}\label{sec:set_detection}

In this section we describe how to obtain a set of ANOVA terms $U$ for approximation. We are sill working with the scattered data $\mathcal X \subset \R^d$ and $\y \in \R^M$, $M \in \N$. The values $\y$ may also contain noise. Our first step is to limit the variable interactions by a superposition threshold $d_s \in \D$ which may have been estimated by known smoothness properties (or different a-priori information) or set to a sensible value if nothing is known. It is also possible to determine an optimal value through cross-validation. We choose the order-dependent parameters $N_{\au}$, $\au = 1,2,\dots,d_s$, cf.\ \eqref{indexsets}, to obtain $\I(U(d,d_s))$ and with the procedure described in \cref{sec:procedure}, the approximation $S(\X, \I(U(d,d_s))) f$. 

From the approximation $S(\X,\I(U(d,d_s))) f$ we can then calculate the global sensitivity indices $\gsi{\u}{S(\X,\I(U(d,d_s))) f }$, $\u \in U(d,d_s)$, and an attribute ranking $r(i)$, $i \in \D$, see \eqref{eq:ranking}. Then we are able to apply the strategies proposed in \cite{PoSc21} to truncate terms from the set $U(d,d_s)$.

One obvious method is the truncation of an entire variable $x_i$, $i \in \D$, if the attribute ranking $r(i)$ shows that its influence is insignificant. Specifically, that would translate to an active set $U^\ast = \{\u \in U(d,d_s) \colon \i \notin \u \}$. This leads to a reduction in dimensionality of the problem and greatly simplifies our model. 

A different method is \emph{active set thresholding} where we chooses a threshold vector $\b\eps \in (0,1)^{d_s}$ and reduce the ANOVA terms to the set \begin{equation*}
	U^\ast(\b\eps) \coloneqq \left\{ \u \in U(d,d_s) \colon \gsi{\u}{S(\X,\I(U(d,d_s))) f } > \eps_{\au}  \right\}.
\end{equation*} Here, $\eps_{\au}$ denotes the $\au$th entry of the vector $\b\eps$. The parameter vector $\b\eps$ allows control over how much of the variance may be sacrificed in order to simplify the model function.

In summary, it is necessary to interpret the information from the approximation $S(\X,\I(U(d,d_s))) f$ and decide on strategies for truncating the set of ANOVA terms. One may also use different strategies to obtain an active set or any combination of the multiple approaches, see e.g.\ \cite{PoSc19b,PoSc21}. Of course, it is also possible to repeat the procedure multiple times, i.e., through cross-validation.

\section{Forest Fire Prevention}

We now apply the previously described method to the data set \cite{forest} from the UC Irvine machine learning repository. The dataset contains information about forest fires in the Montesinho national park in the Trás-os-Montes northeast region of Portugal. The data was collect from 2002 to 2003. Specifically, we have $d = 12$ attributes about the fires and the target variable is the area of the forest that was destroyed by it. If we obtain an efficient model, it can be possible to predict the risk for a future forest fire using parameters that can be easily measured. This information can then be used to prepare appropriate countermeasures. The data set has been thoroughly considered in \cite{CorMor07} and we compare to the results they obtained.

\begin{table}[tbhp]
	\begin{center}
		\begin{tabular}{ccc} 
			\toprule
			\textbf{group} & \textbf{name} & \textbf{description} \\
			\midrule
			\multirow{2}{*}{spatial (S)} & X & x-coordinate (1 to 9) \\
			& Y & y-coordinate (1 to 9) \\ \midrule
			\multirow{2}{*}{temporal (T)} & month & month of the year (1 to 12) \\
			& day & day of the week (1 to 7) \\ \midrule
			\multirow{4}{*}{FWI} & FFMC & FFMC code \\
			& DMC & DMC code \\
			& DC & DC code \\
			& ISI & ISI index \\ \midrule
			\multirow{4}{*}{meteorological (M)} & temp & outside temperature in °C \\
			& RH & outside relative humidity in \% \\
			& wind & outside wind speed in km/h \\
			& rain & outside rain in mm/$\mathrm{m}^2$ \\ \bottomrule
		\end{tabular}
		\caption{Attributes and their corresponding groups}\label{tab:attributes}
	\end{center}
\end{table}

We group the 12 attributes into 4 categories as in \cite{CorMor07}, i.e., spatial, temporal, FWI system, and meteorological data, see \cref{tab:attributes}. The spatial attributes describe the spatial location of the fire in a 9 by 9 grid of our considered region. The temporal attributes are the month of the year and the day of the week when the fire occurred. The forest fire weather index (FWI), cf.\ \cite{Taylor2006}, is the Canadian system for rating fire danger and the datasets collects several components of it. Moreover, four meteorological attributes which are used by the FWI index were selected. The target variable describes the are that was burned by the fire. 

In terms of pre-processing, we apply a Z-score transformation to the the variables and the logarithmic transformation $\log(1+\cdot)$ to the burned area. The Z-score transformation achieves that our data has zero mean and unit variance. The logarithmic transformation on the target is necessary since it shows a positive skew with a large number of fires that have a small size. We denote the data $(\X,\y)$ with $\X = \{\x_1,\x_2,\dots,\x_M\} \subset \R^{12}$, $M = 517$, and $\y \in \R^M$. In the following subsections, we do not use all of the variables, but build models based only on some groups as denoted in \cref{tab:attributes}, e.g., \textbf{STM} says that we use spatial, temporal and meteorological attributes without the FWI. 

\cref{tab:FireRes} shows the overall results of our experiment (ANOVA) combined with the benchmark data from \cite{CorMor07}. Each value, our ANOVA results as well as the others, were obtained by averaging over executing a 10-fold cross-validation 30 times. This results in a total of 300 experiments. We used a superposition threshold of $d_s = 2$, cf.\ \eqref{termset}, and therefore needed to detect optimal choices for the parameters $N_{1}$ and $N_2$ from \eqref{indexsets}, see \cref{tab:FireRes:params}. Every experiment utilized 90\% of the data as training set $(\X_{\mathrm{train}}, \y_{\mathrm{train}})$ and 10\% of the data as test set $(\X_{\mathrm{test}}, \y_{\mathrm{test}})$. The best performing model was selected based on the mean absolute deviation \begin{equation}
	\mathrm{MAD} = \frac{1}{\abs{\X_{\mathrm{test}}}} \sum_{i = 1}^{\abs{\X_{\mathrm{test}}}} \abs{(\tilde{\y})_i - (\y_{\mathrm{test}})_i}
\end{equation} with $\tilde{\y}$ the predictions of our model for the data points in the test set $\X_{\mathrm{test}}$. As a second error measure, we use the root mean square error \begin{equation}
	\mathrm{RMSE} = \frac{1}{\sqrt{\abs{\X_{\mathrm{test}}}}} \sqrt{\sum_{i = 1}^{\abs{\X_{\mathrm{test}}}} \abs{(\tilde{\y})_i - (\y_{\mathrm{test}})_i}^2}.
\end{equation} We are able to outperform the previously applied method for every subset of attributes in both MAD and RMSE error. Notably, the difference in the RMSE that penalizes larger deviations in the burned area stronger than the MAD is much more significant. 

\begin{table}[tbhp]
	\centering
	\begin{tabular}{c|cccc}
		\toprule
		& \multicolumn{4}{c}{attribute selection} \\
		model & S T FWI & S T M & FWI & M \\ \midrule
		Naive & 18.61 (63.7) & 18.61 (63.7) & 18.61 (63.7) & 18.61 (63.7) \\
		MR & 13.07 (64.5) & 13.04 (64.4) & 13.00 (64.5) & 13.01 (64.5) \\
		DT & 13.46 (64.4) & 13.43 (64.6) & 13.24 (64.4) & 13.18 (64.5) \\
		RF & 13.31 (64.3) & 13.04 (64.5) & 13.38 (64.0) & 12.93 (64.4) \\
		NN & 13.09 (64.5) & 13.92 (68.9) & 13.08 (64.6) & 13.71 (66.9) \\
		SVM & 13.07 (64.7) & 13.13 (64.7) & 12.86 (64.7) & 12.71 (64.7) \\
		ANOVA & \textbf{12.75} (\textbf{45.77}) & \textbf{12.81} (\textbf{46.7}) & \textbf{12.76} (\textbf{46.09}) & \underline{\textbf{12.65}} (\underline{\textbf{45.69}}) \\
		\bottomrule
	\end{tabular}
	\caption{MAD and RMSE (in brackets) for the best performing model in the corresponding attribute subset (\underline{underline} - overall best result, \textbf{bold} - best result for this selection).}
	\label{tab:FireRes}
\end{table} 

\begin{table}[tbhp]
	\centering
	\begin{tabular}{c|cccc}
		\toprule
		attribute selection & $N_1$ & $N_2$ & $\abs{I}$ & $\lambda$ \\ \midrule
		S T FWI & 2 & 6 & 149 & $\e^9$ \\
		S T M & 2 & 10 & 261 & $\e^{10}$ \\
		FWI & 2 & 4 & 23 & $\e^8$ \\
		M & 2 & 8 & 47 & $\e^7$ \\
		\bottomrule
	\end{tabular}
	\caption{Optimal parameter choices for the experiments from \cref{tab:FireRes}.}
	\label{tab:FireRes:params}
\end{table} 

While we replicated the setting of \cite{CorMor07} for benchmark purposes, it remains our goal to identify the most important attributes for the detection of forest fires. Therefore, we now use all 12 attributes of the dataset in obtaining our approximation and subsequently interpret the results. \cref{fig:argsi} shows the attribute ranking $r(i)$, $i=1,2,\dots,12$, and the global sensitivity indices $\gsi{\u}{S(\X_{\mathrm{train}},\I(U(12,2)))f}$, $\u \in U(12,2)$, after computing an approximation with $N_1 = N_2 = 2$ and $\lambda = 1.0$. 

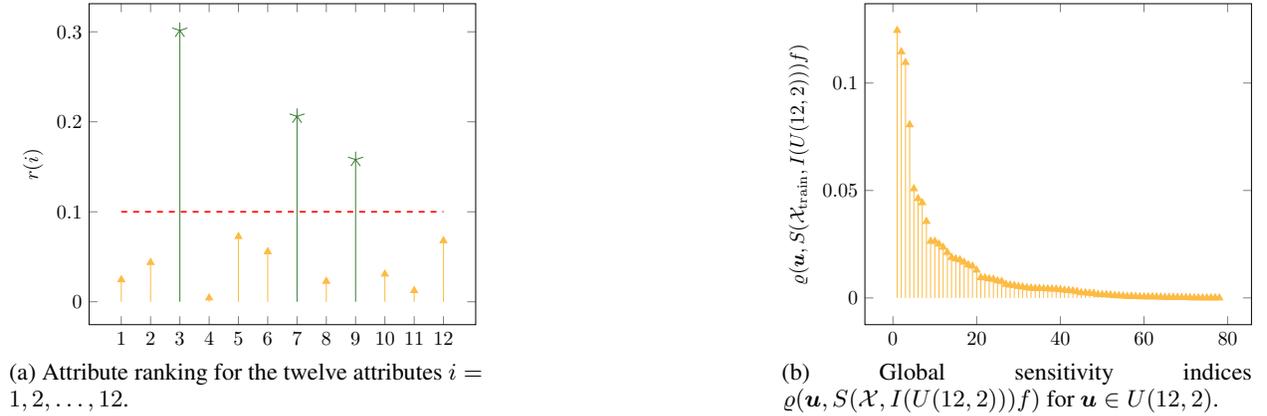
\begin{figure}[tbhp]
	\centering
	\subfloat[Attribute ranking for the twelve attributes $i=1,2,\dots,12$.]{
		\begin{tikzpicture}[scale=0.75]
			\begin{axis}[ylabel={$r(i)$},
				xtick={1,2,3,4,5,6,7,8,9,10,11,12}]
				\addplot+[ycomb, OliveGreen, mark=star, mark options={OliveGreen, scale=2}] plot coordinates
				{(3,0.3014967904608072)(7,0.2060387544136121)(9,0.15801317797015244)};
				\addplot+[ycomb, Dandelion, mark=triangle*, mark options={Dandelion}] plot coordinates
				{(1,0.024526497051189106)(2,0.04366863263065131)(4,0.0042804878297696145)(5,0.07255087425909913)(6,0.05564782523239283)(8,0.022768691792812035)(10,0.03078022703794662)(11,0.012346199727699693)(12,0.06788184159386762)};
				\addplot[red,sharp plot,update limits=false, dashed, thick] 
				coordinates {(1,0.1) (12,0.1)};
			\end{axis}
		\end{tikzpicture}
	} \hfill \subfloat[Global sensitivity indices $\gsi{\u}{S(\X,I(U(12,2)))f}$ for $\u \in U(12,2)$.]{\begin{tikzpicture}[scale=0.75]
		\begin{axis}[ylabel={$\gsi{\u}{S(\X_{\mathrm{train}},I(U(12,2)))f}$},tick label style={/pgf/number format/fixed}]
			\addplot+[ycomb, Dandelion, mark=triangle*, mark options={Dandelion}] plot coordinates
			{(1,0.12446812611216862)(2,0.11455039703620026)(3,0.10958515864195462)(4,0.08057249244795127)(5,0.05082088238095818)(6,0.046235974518989996)(7,0.04430044055852162)(8,0.035602854006107416)(9,0.026384504169043678)(10,0.026326950230226766)(11,0.025049427183250567)(12,0.023564160696417334)(13,0.02117107818762821)(14,0.018841884825971423)(15,0.01815320337027364)(16,0.017719088086606022)(17,0.01662379556996318)(18,0.015366621551640777)(19,0.014781446040523555)(20,0.01292712893807358)(21,0.009441103780675074)(22,0.00934915259936956)(23,0.008901355584038042)(24,0.008683920281175256)(25,0.007947204486423928)(26,0.0077077111821961565)(27,0.006336796362956039)(28,0.005971104283419384)(29,0.005754969748027635)(30,0.005362504560669756)(31,0.005090285995380257)(32,0.00474668869029838)(33,0.00452554291197292)(34,0.004422731266272496)(35,0.0044092099917748505)(36,0.004300130344967892)(37,0.004232630738408847)(38,0.004122083581457995)(39,0.004093371835244823)(40,0.003876128508515062)(41,0.0036055080553272652)(42,0.0034496293492738196)(43,0.0032737140204153617)(44,0.0029774187989129726)(45,0.0023893093146402826)(46,0.0023853308829641035)(47,0.0021827786575713275)(48,0.0021107634640473946)(49,0.001614432309471334)(50,0.0015175167429326807)(51,0.0014969630231066512)(52,0.0012940240587672798)(53,0.0010528502387083943)(54,0.0010135997356330775)(55,0.0008849431920411826)(56,0.0007220200173492534)(57,0.0007123720259231344)(58,0.0006191069743309094)(59,0.000548100276398919)(60,0.0004989151220978082)(61,0.0004924689845466057)(62,0.00048242571523533903)(63,0.0004520768504845259)(64,0.00037269179552028975)(65,0.00025634152118808677)(66,0.00024636741519040294)(67,0.00024209006630711343)(68,0.00022453431638615198)(69,0.00019208761330578806)(70,0.000191447038985034)(71,8.233609919945272e-5)(72,3.2031496740415625e-5)(73,2.5879209742187184e-5)(74,1.6463398428513013e-5)(75,8.64991114778432e-6)(76,7.540403112896306e-6)(77,4.9070205189082945e-6)(78,1.2359833248271945e-7)};
		\end{axis}
\end{tikzpicture}}
	\caption{Explainable approximation results with all twelve attributes using $N_1 = N_2 = 2$ and $\lambda = \e^8$.}\label{fig:argsi}
\end{figure}

The attributes 3, 7, and 9 are clearly the most important. They represent the month of the year (3), the DC code of the FWI (7) and the outside temperature (9). Using only these three attributes and superposition threshold $d_s = 2$, we computed an approximation with $N_1 = 2$, $N_2 = 10$, and $\lambda = \e^8$. The resulting model yielded a MAD of $12.64$ and a RMSE of $45.57$ with 30 times of 10-fold cross validation as before. In summary, we know that the most important information of our problem is contained in only three attributes and we also obtained a better performing model using only those three attributes.

\section*{Acknowledgments}
The authors thank their colleagues in the research group SAlE for valuable discussions on the contents of this paper. Daniel Potts acknowledges funding by Deutsche Forschungsgemeinschaft (German Research Foundation) -- Project--ID 416228727 -- SFB 1410. Michael Schmischke is supported by the German Federal Ministry of Education and Research grant 01$|$S20053A. 

\bibliographystyle{abbrv}
\bibliography{refs} 

\end{document}